\documentclass[12pt,english]{article}
\usepackage[T1]{fontenc}
\usepackage[latin9]{inputenc}
\usepackage{geometry}
\usepackage{longtable}
\usepackage{amsmath}
\usepackage{amssymb}
\usepackage{setspace}
\usepackage{esint}
\makeatletter

\providecommand{\tabularnewline}{\\}

\newtheorem{theorem}{Theorem}

\newtheorem{condition}{Condition}

\newtheorem{corollary}{Corollary}

\newtheorem{lemma}{Lemma}

\newenvironment{proof}[1][Proof]{\textbf{#1.} }{\ \rule{0.5em}{0.5em}}

\makeatother

\usepackage{babel}

\makeatother

\usepackage{babel}

\makeatother

\usepackage{babel}

\makeatother

\usepackage{babel}
\begin{document}

\title{Consistency Results for Stationary Autoregressive Processes with
Constrained Coefficients}

\author{Alessio Sancetta\thanks{Acknowledgements: I am grateful to Luca Mucciante for insightful conversations.
E-mail: <asancetta@gmail.com>, URL: <http://sites.google.com/site/wwwsancetta/>.
Address for correspondence: Department of Economics, Royal Holloway
University of London, Egham TW20 0EX, UK }\\
}
\maketitle
\begin{abstract}
We consider stationary autoregressive processes with coefficients
restricted to an ellipsoid, which includes autoregressive processes
with absolutely summable coefficients. We provide consistency results
under different norms for the estimation of such processes using constrained
and penalized estimators. As an application we show some weak form
of universal consistency. Simulations show that directly including
the constraint in the estimation can lead to more robust results. 

\textbf{Key Words:} consistency, empirical process, ridge regression,
reproducing kernel Hilbert space, universal consistency. 
\end{abstract}

\section{Introduction}

It is common to impose constraints on the decay rate of the autoregressive
coefficients in order to derive results amenable to estimation for
the purpose of prediction. At minimum, these constraints tend to require
that the AR coefficients are absolutely summable. Then, a natural
approach when dealing with high order autoregressive models is to
consider sieve estimation. Sieve estimation of infinite AR models
has been considered by various authors. For universal consistency,
Sch\"{a}fer (2002) derived perhaps the strongest result possible.
Györfi and Sancetta (2015) review some of these results. For convergence
in probability, various authors have considered infinite AR models
and its applications, e.g. Bühlmann (1997), and Kreiss et al. (2011).
Additional references can be found in the cited papers. 

Here, we constraint the autoregressive coefficients to lie in an infinite
dimensional ellipsoid such that coefficients associated to higher
order lags decay fast. Then, we can exploit the fact that the ellipsoid
is compact under the $\ell_{2}$ norm in order to derive asymptotic
results. The conditions essentially require the autoregressive coefficients
to be absolutely summable. We shall see that the vector of autoregressive
coefficients can be seen as an element in a Reproducing Kernel Hilbert
Space (RKHS) when $\ell_{2}$ is equipped with a suitable inner product.
This allows us to exploit all the existing machinery for estimation
in RKHS and build on it (Steinwart and Chirstmann, 2008, for a comprehensive
review) . The main ingredient is penalized least square estimation.
We also consider the constrained least square problem. Penalized and
constrained estimation are dual problems for specific values of the
penalty coefficient. Our result establishes the relation between the
two problems and the consistency rates. In general, they can lead
to different consistency results under different norms. One norm is
the usual Euclidean norm of the vector of coefficients while the other
is the norm of the RKHS. We show that consistency under the latter
has important implications for prediction problems. 

In general, unlike existing results we are able to establish consistency
as both the autoregressive order and the sample size go to infinity
with no constraint on the rates. Existing results use the machinery
of method of sieve, hence they require the autoregressive order to
go to infinity in a controlled way. As already mentioned, we are able
to avoid this restriction because the ellipsoid is compact under the
Euclidean norm.

The plan for the paper is as follows. Section \ref{Section_estimationMethod}
reviews the estimation method and presents the consistency results.
A numerical example is provided in Section \ref{Section_simulations}.
Section \ref{Section_furhterRemarks} mentions extensions to other
processes such as vector autoregressive processes (VAR). The proof
of the consistency results is long and is given in Section \ref{Section_proofs}.

\section{Estimation Method\label{Section_estimationMethod}}

We restrict attention to the infinite order autoregressive process
\begin{equation}
Y_{t}=\sum_{k=1}^{\infty}\varphi_{k}Y_{t-k}+\varepsilon_{t}\label{EQ_trueModel}
\end{equation}
for some mean zero independent identically distributed (i.i.d.) sequence
$\left(\varepsilon_{t}\right)_{t\in\mathbb{Z}}$ and unknown coefficients
$\varphi_{k}$'s. This paper considers estimators of the above under
the condition that $\sum_{k=1}^{\infty}\left|\varphi_{k}\right|\leq\bar{\varphi}<\infty$. 

In a finite sample, the above model can only be approximated by the
finite dimensional model 
\[
Y_{t}=\sum_{k=1}^{K}b_{k}Y_{t-k}+\varepsilon_{t}
\]
with $K\rightarrow\infty$. While this is essentially a sieve we do
not necessarily require $K$ to be of smaller order than the sample
size. Here, we restrict the coefficients in an ellipsoid to be defined
as follows. Let $\lambda_{k}$'s be positive constants such that $\lambda_{k}\asymp k^{\lambda}$
for $\lambda>0$, where $\asymp$ means that the left hand side (l.h.s.)
and the right hand side (r.h.s.) are proportional. Define the ellipsoid
as 
\begin{equation}
\mathcal{E}_{K}\left(B\right):=\left\{ b\in\mathbb{R}^{\infty}:\sum_{k=1}^{\infty}b_{k}^{2}\lambda_{k}^{2}\leq B^{2},\,b_{k}=0\text{ for }k>K\right\} .\label{EQ_EBset}
\end{equation}
Given that the $\lambda_{k}$'s are increasing, the $b_{k}$'s need
to be smaller in absolute values as $k$ increases. Write $\mathcal{E}\left(B\right)=\bigcup_{K>0}\mathcal{E}_{K}\left(B\right)$
for the ellipsoid where all coefficients can be non-zero, $\mathcal{E}_{K}=\bigcup_{B<\infty}\mathcal{E}_{K}\left(B\right)$
and $\mathcal{E}=\bigcup_{B<\infty}\mathcal{E}\left(B\right)$, so
for example $\mathcal{E}=\left\{ b\in\mathbb{R}^{\infty}:\sum_{k=1}^{\infty}b_{k}^{2}\lambda_{k}^{2}<\infty\right\} $
is the ellipsoid that is restricted to have finite but decreasing
principal axes. The following condition will be imposed on the ellipsoid.

\begin{condition}\label{Condition_ellipsoid}The sequence $\left(Y_{t}\right)_{t\in\mathbb{Z}}$
follows the process (\ref{EQ_trueModel}) with $\varphi\in\mathcal{E}$
and $\lambda_{k}\asymp k^{\lambda}$, where $\lambda>1/2$. Moreover,
$1-\sum_{k=1}^{\infty}\varphi_{k}z^{k}=0$ only for $z$ outside the
unit circle. The innovations $\left(\varepsilon_{t}\right)_{t\in\mathbb{Z}}$
are independent identically distributed with finite fourth moment.\end{condition}

Throughout, when writing $\mathcal{E}_{K}\left(B\right)$ and similar
quantities, it is understood that the $\lambda_{k}$'s are as in Condition
\ref{Condition_ellipsoid}. The following is stated for convenience. 

\begin{lemma}\label{Lemma_bDecay}If $b\in\mathcal{E}\left(B\right)$
then, $b_{k}\lesssim k^{-\left(2\lambda+1\right)/2}/\ln^{1+\epsilon}\left(1+k\right)$
for some $\epsilon>0$, where $\lesssim$ is inequality up to a fixed
 absolute multiplicative constant.\end{lemma}

In consequence, Condition \ref{Condition_ellipsoid} implies absolutely
summable autoregressive coefficients. Note that absolute summability
would just require $\lambda\geq1/2$ in Condition \ref{Condition_ellipsoid}
rather than $\lambda>1/2$, hence the condition we use is a bit more
restrictive.  The following states additional properties of the model. 

\begin{lemma}\label{Lemma_ergodicityModel} Under Condition \ref{Condition_ellipsoid},
$\left(Y_{t}\right)_{t\in\mathbb{Z}}$ is stationary and ergodic with
absolutely summable autocovariance function and $\mathbb{E}Y_{t}^{4}<\infty$.\end{lemma}

It is well known that for the AR process, $1-\sum_{k=1}^{\infty}\varphi_{k}z^{k}=0$
only for $z$ outside the unit circle if the autocovariance function
is absolutely summable and the spectral density is strictly positive
and continuous (Kreiss et al., 2011, Corollary 2.1).

Note that there are processes (even Gaussian) that satisfy Condition
\ref{Condition_ellipsoid}, but fail to be beta mixing (Doukhan, 1995,
Theorem 3, p.59). The beta mixing assumption is often conveniently
used when proving convergence using methods from empirical process
theory. Alas, it cannot be used here. 

\subsection{Estimation and Consistency}

The goal is to find an estimator for $\varphi$. We consider two approaches:
constrained least square and penalized least square. By duality, the
two can be made to be equivalent by suitable choice of the penalty
parameter. However, in the constrained case, the penalty turns out
to be sample dependent, while in penalized estimation this it not
necessarily the case. 

To avoid notational trivialities, suppose that the sample size is
$N=n+K$. This will be assumed without further notice throughout the
paper. In particular, our sample is $Y_{-\left(K-1\right)},Y_{-\left(K-2\right)},...,Y_{0},Y_{1},...,Y_{n}$.
This also stresses the fact that $n$ and $K$ can go to infinity
at different rates.

In the constrained problem, we estimate $b\in\mathcal{E}_{K}\left(B\right)$.
The constrained estimator is defined as
\begin{equation}
b_{n}=\arg\inf_{b\in\mathcal{E}_{K}\left(B\right)}\frac{1}{n}\sum_{t=1}^{n}\left(Y_{t}-\sum_{k=1}^{\infty}b_{k}Y_{t-k}\right)^{2}\label{EQ_ConstrainedObjFunc}
\end{equation}
Of course, in the above, $\sum_{k=1}^{\infty}b_{k}Y_{t-k}=\sum_{k=1}^{K}b_{k}Y_{t-k}$
if $b\in\mathcal{E}_{K}\left(B\right)$. 

In the penalized problem, we estimate $b\in\mathcal{E}_{K}$, but
introduce the penalty parameter $\tau>0$. The penalized estimator
is defined as 
\begin{equation}
b_{n,\tau}:=\arg\inf_{b\in\mathcal{E}_{K}}\frac{1}{n}\sum_{t=1}^{n}\left(Y_{t}-\sum_{k=1}^{\infty}b_{k}Y_{t-k}\right)^{2}+\tau\sum_{k=1}^{\infty}\lambda_{k}^{2}b_{k}^{2},\label{EQ_penalizedObjFunc}
\end{equation}
where the $\lambda_{k}$'s are from the definition of $\mathcal{E}$.
By use of the Lagrangian, we can always rewrite (\ref{EQ_ConstrainedObjFunc})
as (\ref{EQ_penalizedObjFunc}) for suitable choice of $\tau$, i.e.
there is a $\tau=\tau_{B,n}$ ($\tau=0$ if the constraint it not
binding) such that $b_{n,\tau}=b_{n}$. 

Both problems can be reformulated in matrix form using the Lagrangian.
Let $X$ be the $n\times K$ dimensional matrix with $\left(t,k\right)^{th}$
entry equal to $Y_{t-k}$ and $Y$ be the $n$-dimensional vector
with $t^{th}$ entry $Y_{t}$. Also, let $\Lambda$ be the $K\times K$
diagonal matrix with $k^{th}$ diagonal entry equal to $\lambda_{k}$.
The estimator for either (\ref{EQ_ConstrainedObjFunc}) or (\ref{EQ_penalizedObjFunc})
is found by minimizing the penalized least square criterion with respect
to (w.r.t.) $\tilde{b}\in\mathbb{R}^{K}$,
\begin{equation}
\frac{1}{n}\left(Y-X\tilde{b}\right)^{T}\left(Y-X\tilde{b}\right)+\tau\tilde{b}^{T}\Lambda^{2}\tilde{b}\label{EQ_objFunc}
\end{equation}
where for (\ref{EQ_ConstrainedObjFunc}) $\tau$ is chosen so that
the constraint $\tilde{b}^{T}\Lambda\tilde{b}\leq B^{2}$ is satisfied.
In this latter case, $\tau$ is necessarily random because the constraint
needs to be satisfied in sample. Here the tilde in $\tilde{b}$ is
used to remind us that in the matrix formulation, $b$ is truncated
to be a $K$ dimensional vector, as all entries larger than $K$ are
zero by definition of $\mathcal{E}_{K}$. The solution is the usual
ridge regression estimator $\tilde{b}_{n,\tau}:=\left(X^{T}X+\tau\Lambda^{2}\right)^{-1}X^{T}Y$. 

For problem (\ref{EQ_penalizedObjFunc}), $\tau=\tau_{n}$ can go
to zero in a controlled way. For problem (\ref{EQ_ConstrainedObjFunc}),
$\tau=\tau_{B,n}\geq0$ must be chosen so that the constraint is satisfied.
Such $\tau_{B,n}$ is zero if the constraint is binding, and zero
otherwise. This is equivalent to replacing $\tau\tilde{b}^{T}\Lambda^{2}\tilde{b}$
with $\left(\tilde{b}^{T}\Lambda^{2}\tilde{b}-B^{2}\right)$ in (\ref{EQ_objFunc}),
and minimizing the so modified objective function (\ref{EQ_objFunc})
w.r.t. $\tilde{b}$ and $\tau\geq0$. The minimizer w.r.t. $\tau$
is $\tau_{B,n}$. 

All vectors are in $\mathbb{R}^{\infty}$, though only the first $K$
elements might be non-zero. The exception is when we use a tilde,
as in (\ref{EQ_objFunc}). For $b_{n}$ in (\ref{EQ_ConstrainedObjFunc}),
the Euclidean norm of $b_{n}-\varphi$ becomes $\left|b_{n}-\varphi\right|_{2}=\left(\sum_{k=1}^{K}\left|b_{nk}-\varphi_{k}\right|^{2}+\sum_{k>K}\left|\varphi_{k}\right|^{2}\right)^{1/2}$

It is worth noting that the ellipsoid $\mathcal{E}\subset\ell_{2}$
is a RKHS generated by the kernel $C\left(k,l\right)=\sum_{v=1}^{\infty}\lambda_{v}^{-2}\delta_{v,k}\delta_{v,l}$
where $\delta_{v,l}$ is the Kronecker's delta, i.e. $\delta_{v,l}=1$
if $v=l$ and zero otherwise. The inner product $\left\langle \cdot,\cdot\right\rangle _{\mathcal{E}}$
is defined to satisfy the reproducing kernel property $\left\langle C\left(\cdot,l\right),C\left(\cdot,k\right)\right\rangle _{\mathcal{E}}=C\left(k,l\right)$.
Hence for $a,b\in\mathcal{E}$, $b_{k}=\left\langle b,C\left(\cdot,k\right)\right\rangle _{\mathcal{E}}$
and $\left\langle a,b\right\rangle _{\mathcal{E}}=\sum_{v=1}^{\infty}\lambda_{v}^{2}a_{v}b_{v}$.
The norm induced by the inner product is $\left|\cdot\right|_{\mathcal{E}}$
such that for any vector $b\in\mathbb{R}^{\infty}$, $\left|b\right|_{\mathcal{E}}^{2}=\sum_{k=1}^{\infty}\lambda_{k}^{2}b_{k}^{2}$.
This norm strictly dominates the Euclidean norm. The fact that $\mathcal{E}\left(1\right)$
is compact under the Euclidean norm is a consequence of the fact that
$\mathcal{E}$ is a RKHS (Li and Linde, 1999) and sharp asymptotics
can be derived by related means (Graf and Luschgy, 2004).

Once we realize such compactness, it becomes clear that it might be
possible to estimate infinite AR processes under no restriction on
the number of estimated coefficients. We show that this conjecture
is true. We also establish convergence rates. Moreover, we want to
clearly address the relation between constrained and penalized estimation. 

The best approximation $\varphi_{K}\in\mathcal{E}_{K}$ to $\varphi$
minimizes the population mean square error 
\begin{equation}
\varphi_{K}=\arg\inf_{b\in\mathcal{E}_{K}}\mathbb{E}\left(Y_{1}-\sum_{k=1}^{\infty}b_{k}Y_{1-k}\right)\label{EQ_populationKEstimator}
\end{equation}
Despite the abuse of notation, do not confuse $\varphi_{K}$ with
the $K^{th}$ entry in $\varphi$.

\begin{theorem}\label{Theorem_consistency}Suppose that Condition
\ref{Condition_ellipsoid}, and $n,\,K\rightarrow\infty$ hold. 
\begin{enumerate}
\item (Consistency of Constrained Estimator) If $\varphi\in\mathcal{E}\left(B\right)$There
is a random $\tau=\tau_{B,n}$ such that $\tau=O_{p}\left(n^{-1/2}\right)$,
$b_{n,\tau}=b_{n}$ and if $\varphi\in\mathcal{E}\left(B\right)$,
$\left|b_{n}-\varphi\right|_{2}=O_{p}\left(n^{-\frac{1}{2}\left(\frac{2\lambda-\epsilon}{2\lambda-\epsilon+1}\right)}+K^{-\lambda}\right)$
for any $\epsilon\in\left(0,2\lambda-1\right)$. 
\item (Consistency of Penalized Estimator) Consider possibly random $\tau=\tau_{n}$
such that $\tau\rightarrow0$ and $\tau n^{1/2}\rightarrow\infty$
in probability. There is a finite $B$ such that $\varphi\in\mathrm{int}\left(\mathcal{E}\left(B\right)\right)$,
$\left|b_{n,\tau}\right|_{\mathcal{E}}<B$ eventually in probability
and $\left|b_{n,\tau}-\varphi\right|_{\mathcal{E}}\rightarrow0$ in
probability. 
\item (Approximation Error in $\mathcal{E}$) There is an $\epsilon>0$
such that $\left|\varphi-\varphi_{K}\right|_{\mathcal{E}}=O\left(\left(\ln K\right)^{-\left(1+\epsilon\right)}\right)$.
Suppose the $k^{th}$ entry $\varphi_{k}$ in $\varphi$ satisfies
$\left|\varphi_{k}\right|\lesssim k^{-\nu}$ with $\nu>\left(2\lambda+1\right)/2$
for all $k$ large enough. Then $\left|\varphi-\varphi_{K}\right|_{\mathcal{E}}=O\left(K^{\left(2\lambda+1-2\nu\right)/2}\right)$. 
\item (Estimation Error in $\mathcal{E}$) If $\left(\tau+n^{-1/2}\right)=O_{p}\left(K^{-2\lambda}\right)$,
then $\left|b_{n,\tau}-\varphi_{K}\right|_{\mathcal{E}}=O_{p}\left(n^{-1/4}K^{\lambda}\right)$ 
\item (Difference Between Norms) There is $K\rightarrow\infty$ and $\tau=O_{p}\left(n^{-1/2}\right)$
such that $\left|b_{n,\tau}-\varphi\right|_{2}\rightarrow0$ in probability,
but $\left|b_{n,\tau}-\varphi\right|_{\mathcal{E}}$ does not converge
to zero in probability. 
\end{enumerate}
\end{theorem}

Point 1 in the theorem establishes the link between constrained and
penalized estimation by finding the rate of decay of the ridge penalty
so that (\ref{EQ_ConstrainedObjFunc}) and (\ref{EQ_penalizedObjFunc})
are the same. It also establishes the convergence rate of (\ref{EQ_ConstrainedObjFunc})
towards the true $\varphi$ in terms of $\lambda$ (recall $\lambda_{k}\asymp k^{\lambda}$
in Condition \ref{Condition_ellipsoid}). This rate does not constrain
the number of lags used once we constrain $\varphi\in\mathcal{E}\left(B\right)$.
For the finite dimensional case we trivially recover the root-n convergence
by letting $\lambda\rightarrow\infty$. 

Point 2 says that if we use the penalized estimation and the penalty
does not go to zero too fast (i.e. strictly slower than in Point 1)
we can expect (\ref{EQ_penalizedObjFunc}) to be contained in a ball
in $\mathcal{E}$ that contains the true parameter with probability
going to one. Moreover, (\ref{EQ_penalizedObjFunc}) is consistent
under the norm $\left|\cdot\right|_{\mathcal{E}}$. 

Point 3 is concerned with the approximation error of (\ref{EQ_populationKEstimator})
in the RKHS norm. This error might go to zero at a logarithmic rate.
However, if the true coefficients decay fast, then we can have polynomial
convergence rate. 

Point 4 restricts the way we let $K\rightarrow\infty$ in order to
derive convergence rates of the estimation error under the norm $\left|\cdot\right|_{\mathcal{E}}$. 

Point 5 establishes an additional insight between the convergence
under the Euclidean norm and the RKHS norm in terms of the penalty.
A ``slowly convergent'' penalty is necessary for convergence under
$\left|\cdot\right|_{\mathcal{E}}$. Hence, this also shows that the
constrained estimator (whose penalty is $\tau=\tau_{B,n}=O_{p}\left(n^{-1/2}\right)$
when $\varphi\in\mathcal{E}\left(B\right)$) cannot be consistent
in the norm $\left|\cdot\right|_{\mathcal{E}}$ in general. This happens
when choosing a rather large $K$ that leads to a binding constraint
for (\ref{EQ_ConstrainedObjFunc}). 

As corollary to Points 3 and 4 in Theorem \ref{Theorem_consistency},
we have the following.

\begin{corollary}\label{Corollary_consistency}Suppose Condition
\ref{Condition_ellipsoid} holds, $K\rightarrow\infty$ and $\tau=O_{p}\left(K^{-2\lambda}\right)$. 
\begin{enumerate}
\item Choose $K\asymp n^{\kappa}$ for some $\kappa\in\left(0,1/4\right)$.
Then, there is an $\epsilon>0$ such that $\left|b_{n,\tau}-\varphi\right|_{\mathcal{E}}=O_{p}\left(\left(\ln K\right)^{-\left(1+\epsilon\right)}\right)$. 
\item Suppose the $k^{th}$ entry $\varphi_{k}$ in $\varphi$ satisfies
$\left|\varphi_{k}\right|\lesssim k^{-\nu}$ with $\nu>\left(2\lambda+1\right)/2$
for all $k$ large enough. Choose $K\asymp n^{\frac{1}{2\left(2\nu-1\right)}}$.
Then, $\left|b_{n,\tau}-\varphi\right|_{\mathcal{E}}=O_{p}\left(n^{-\frac{2\nu-\left(2\lambda+1\right)}{4\left(2\nu-1\right)}}\right)$. 
\end{enumerate}
\end{corollary}

Corollary \ref{Corollary_consistency} imposes additional restrictions
in order to improve on the statement of Point 2 in Theorem \ref{Theorem_consistency}
by giving rates of convergence. These rates are not tight as they
require $K=o\left(n\right)$ unlike Point 2 in Theorem \ref{Theorem_consistency}.
However, they are useful in applications (e.g. Section \ref{Section_uniformConsistency}). 

Sieve estimators are often consistent under the sole condition that
the number of components (here $K$) is of smaller order of magnitude
than the sample size $n$. In Point 1 of Theorem \ref{Theorem_consistency},
we have shown that this is not required. Recall that $N=n+K$ is the
sample size. We can have $K=O\left(N\right)$ as long as $n\rightarrow\infty$.
Of course, we require knowledge concerning the magnitude of the coefficients.
Such knowledge is usually assumed in the literature in order to bound
the approximation error. 

In practice the fact that we allow $K=O\left(N\right)$ might sound
irrelevant. However, the asymptotic results can be seen as suggesting
that, once we set the constraint, the procedure used here can be more
robust to lag choice. We show this in the simulation in Section \ref{Section_simulations}.

\subsubsection{Application to Optimal Forecasting and Universal Consistency\label{Section_uniformConsistency}}

Define $X_{t}\left(a\right)=\sum_{k=1}^{\infty}a_{k}Y_{t-k}$ for
any $a\in\mathbb{R}^{\infty}$. The expectation of $Y_{t}$ conditioning
on the infinite past $\left(Y_{t-s}\right)_{s>0}$ is $X_{t}\left(\varphi\right)$.
As an application of Theorem \ref{Theorem_consistency} consider the
following problem. Show that 
\[
\sup_{t\in\mathcal{T}}\left|X_{t}\left(\varphi\right)-X_{t}\left(b_{n,\tau}\right)\right|\rightarrow0
\]
in probability where $\mathcal{T}=\left(0,\infty\right)$ or $\left(0,n\right)$
($b_{n,\tau}$ in (\ref{EQ_penalizedObjFunc})). Hence, we want $X_{t}\left(b_{n,\tau}\right)$
to be close to the conditional expectation of $Y_{t}$ uniformly in
$t\in\mathcal{T}$ , which is even more general than considering a
moving target. The norm $\left|\cdot\right|_{\mathcal{E}}$ is useful
because the previous display can be written as 
\begin{align}
\sup_{t\in\mathcal{T}}\left|X_{t}\left(\varphi-b_{n,\tau}\right)\right| & \lesssim\left|\varphi-b_{n,\tau}\right|_{\mathcal{E}}\sup_{t\in\mathcal{T}}\left(\sum_{k=1}^{\infty}\left(\frac{Y_{t-k}}{k^{\lambda}}\right)^{2}\right)^{1/2}.\label{EQ_weakUConsistency}
\end{align}
To obtain the inequality, we have multiplied and divided each term
in the sum (on the l.h.s.) by $\lambda_{k}$ and then used the Cauchy-Schwarz
inequality and Condition \ref{Condition_ellipsoid} to set $\lambda_{k}\asymp k^{\lambda}$. 

We have that $\left|\varphi-b_{n,\tau}\right|_{\mathcal{E}}=O_{p}\left(\epsilon_{n}\right)$
in probability, where $\epsilon_{n}\rightarrow0$ at rate which depends
on Theorem \ref{Theorem_consistency}. Then, if 
\begin{equation}
\sup_{t\in\mathcal{T}}\left(\sum_{k=1}^{\infty}\left(\frac{Y_{t-k}}{k^{\lambda}}\right)^{2}\right)^{1/2}=o_{p}\left(\epsilon_{n}^{-1}\right),\label{EQ_uniformVariablesweightedSum}
\end{equation}
we have shown that (\ref{EQ_weakUConsistency}) goes to zero in probability.
This is a weak form of universal consistency because the convergence
is in probability rather than almost surely. On the positive side,
the convergence holds for a variety of processes and circumstances. 

If $\mathcal{T}=\left(0,\infty\right)$ then (\ref{EQ_uniformVariablesweightedSum})
is almost surely finite if the random variables are bounded, and (\ref{EQ_weakUConsistency})
goes to zero in probability using Point 2 in Theorem \ref{Theorem_consistency}.

If $\mathcal{T}=\left(0,n\right)$, we can use the bound 
\[
\left(\mathbb{E}\sup_{t\in\left(0,n\right)}\sum_{k=1}^{\infty}Y_{t-k}^{2}k^{2\lambda}\right)^{1/2}\leq n^{1/\left(2p\right)}\sup_{t\in\left(0,n\right)}\left(\mathbb{E}\sum_{k=1}^{\infty}Y_{t-k}^{2p}k^{2\lambda p}\right)^{1/\left(2p\right)}
\]
when the variables are $2p$ integrable. If $p$ is such that $n^{1/\left(2p\right)}=o\left(\epsilon_{n}^{-1}\right)$,
then the r.h.s. of (\ref{EQ_weakUConsistency}) goes to zero in probability.
If $Y_{t}$ has moment generating function the r.h.s. of the above
display is $O\left(\ln n\right)$. Either way, to find $\epsilon_{n}$
we can use Corollary \ref{Corollary_consistency}. Note that the argument
is unchanged if $\mathcal{T}=\left(0,c_{n}\right)$ for any $c_{n}\asymp n$.

Theorem \ref{Theorem_consistency} can also be applied to the less
ambitious problem: show that 
\[
\lim_{K\rightarrow\infty}\sup_{t\in\mathcal{T}}\left|X_{t}\left(\varphi_{K}\right)-X_{t}\left(b_{n,\tau}\right)\right|\rightarrow0
\]
in probability. In this case we want to forecast as well as the increasingly
best approximation of the conditional expectation of $Y_{t}$, uniformly
in $t\in\mathcal{T}$. Point 4 in Theorem \ref{Theorem_consistency}
is suited for this problem. 

\subsection{Choice of $B$ in Practice\label{Section_choiceB}}

The parameter $B$ can be chosen to minimize some cross-validated
prediction error estimate (beware of cross-validation in a time series
context, e.g. Györfi et al., 1990, Burman and Nolan, 1992, Burman
et al., 1994, for discussions and applicability). Alternatively, one
can choose $B$ to minimize some penalized loss function such as 
\begin{equation}
\ln\hat{\sigma}_{B}^{2}+\frac{2\mathrm{df}\left(B\right)}{n}\label{EQ_Akaike}
\end{equation}
where $\mathrm{df}\left(B\right)=\mathrm{Trace}\left(\left(X^{T}X+\tau_{B,n}n\Lambda^{2}\right)^{-1}X^{T}X\right)$
and $\tau_{B,n}$ is the solution of $\tilde{b}_{n}^{T}\Lambda^{2}\tilde{b}_{n}\leq B$,
using the notation in (\ref{EQ_objFunc}). Here, $\hat{\sigma}_{B}^{2}$
is the sample variance of the residuals from the estimation. If the
constraint is binding, $\tau_{B,n}$ solves
\begin{equation}
Y^{T}X\left(X^{T}X+\tau_{B,n}n\Lambda^{2}\right)^{-2}X^{T}Y=B^{2}.\label{EQ_lagrangeMultiplierImplicit}
\end{equation}
This $\tau_{B,n}$ is then used to compute $\mathrm{df}\left(B\right)$,
which is the effective number of degrees of freedom implied by $B$
(Hastie et al., 2009) 

\section{Numerical Example\label{Section_simulations}}

Asymptotic results are of interest on their own, but it is also of
interest to understand the scope of applicability in practice. As
a benchmark, we use predictions based on an AR model where the lag
length is chosen by Akaike's Information Criterion (AIC).

\subsection{Simulated True Models}

One thousand data samples are simulated from (\ref{EQ_trueModel}).
The sample size is $N=1000$. A warm up sample of 1000 observations
is used to reduce any dependence on the starting value. We also simulate
a testing sample of $1000$ observations to approximate the mean square
error (MSE). We consider different specifications for $\varphi$ in
(\ref{EQ_trueModel}) including long memory in order to see how the
procedure works when the true model is not in $\mathcal{E}$. In this
case, an approximation error is incurred. 

\paragraph{Short Memory}

In (\ref{EQ_trueModel}), the errors are i.i.d. standard normal and
the $\varphi_{k}$'s are chosen to be $\varphi_{k}=\bar{\varphi}k^{-1/2}/\left(\sum_{k=1}^{K_{0}}k^{-1/2}\right)$
, where $\bar{\varphi}=0.75,\,0.99$. A higher value for $\bar{\varphi}$
leads to a more persistent behaviour. By construction, for both values
of $\bar{\varphi}$, the model appears to generate cycles because
the roots of $1-\sum_{k=1}^{K_{0}}\varphi_{k}z^{k}=0$ are outside
the unit circle, but complex. We shall have different values for $K_{0}\in\left\{ 100,1000\right\} $.
Given the finite number of lags the coefficients are automatically
in $\mathcal{E}$. 

\paragraph{Long Memory Model}

The model is an ARFIMA
\begin{equation}
Y_{t}=\sum_{k=1}^{K_{0}}\varphi_{k}Y_{t-k}+\left(1-L\right)^{-d}\left(\sum_{l=0}^{L}\theta_{l}\varepsilon_{t-l}\right)\label{EQ_longMemoryModel}
\end{equation}
where the $\varphi_{k}$'s are as in the previous paragraph. The MA
polynomial is $\theta_{l}=\left(1-0.1l\right)$ with $L=5$. The coefficient
of fractional integration $d=0.49$. Hence, the model is stationary,
but exhibits long memory. 

\subsection{Estimation and Results}

The parameter's estimates are obtained from (\ref{EQ_objFunc}) with
$\lambda_{k}=k^{-0.501}$. The benchmark is an AR model with lag length
chosen to minimize AIC. Denote the number of lags chosen using AIC
by $K_{AIC}$. We compare this to a model estimated using more lags,
but with coefficients constrained in $\mathcal{E}_{K}\left(B\right)$.
In particular, $K=2K_{AIC}$ and $4K_{AIC}$ with $B$ chosen as outlined
in Section \ref{Section_choiceB} . The goal is to verify whether
the procedure is robust to lag choice. AIC is known to choose large
models. We use even larger models, and verify whether we are able
to obtain sensible results. 

The results in Table \ref{Table_simulations} show the improvement
in MSE of the constrained procedure over AIC. Table \ref{Table_simulations}
shows that the procedure is robust against lag choice. This becomes
evident in the long memory case. The larger model ($4K_{AIC})$ leads
to relatively better performance when the true model exhibits persistency
as (\ref{EQ_longMemoryModel}).

\begin{longtable}{cccccc}
\caption{Simulation Results. For Short Memory the process is as in (\ref{EQ_trueModel})
with number of true AR coefficients equal to $K_{0}$ and AR coefficients
satisfying $\varphi_{k}=\bar{\varphi}k^{-1/2}/\left(\sum_{k=1}^{K_{0}}k^{-1/2}\right)$
, where $\bar{\varphi}=0.75,\,0.99$. For Long Memory, the process
is as in (\ref{EQ_longMemoryModel}). Entries denote the MSE improvement
relative to the MSE of a model with lag length $K_{AIC}$ chosen using
AIC. MSE in the numerator in the calculation of the relative improvement
is computed using lag length $2K_{AIC}$ and $4K_{AIC}$ and constraining
the coefficients in $\mathcal{E}\left(B\right)$ where $B$ is chosen
as described in Section \ref{Section_choiceB}.}
\label{Table_simulations}\tabularnewline
 &  &  &  &  & \tabularnewline
$K_{0}=$ & 100 &  & \multicolumn{2}{c}{1000} & \multicolumn{1}{c}{}\tabularnewline
 & $2K_{AIC}$ & $4K_{AIC}$ & $2K_{AIC}$ & $4K_{AIC}$ & \tabularnewline
 & \multicolumn{4}{c}{Short Memory} & \tabularnewline
$\bar{\varphi}=0.75$ & 0.99 & 0.99 & 0.99 & 0.99 & \tabularnewline
$\bar{\varphi}=0.99$ & 0.99 & 0.99 & 0.99 & 0.99 & \tabularnewline
 & \multicolumn{4}{c}{Long Memory} & \tabularnewline
$\bar{\varphi}=0.75$ & 0.93 & 0.88 & 0.94 & 0.88 & \tabularnewline
$\bar{\varphi}=0.99$ & 0.93 & 0.88 & 0.94 & 0.88 & \tabularnewline
\end{longtable}

\section{Further Remarks\label{Section_furhterRemarks}}

It is simple to impose linear restrictions on the coefficients of
either the constrained or penalized estimator. A natural example is
positivity. This is the case if we wish to estimate ARCH models of
large orders. Under ARCH restrictions, the squared returns follow
an AR process. The estimator does not have a closed form expression,
but it is just the solution of a quadratic programming problem. Another
extension pertains to vector autoregressive processes 
\begin{equation}
Y_{t}=\sum_{k=1}^{\infty}\Phi_{k}Y_{t-k}+\varepsilon_{t}\label{EQ_VARModel}
\end{equation}
where now the variables and innovations are $L$ dimensional vectors
and we use the capital $\Phi_{k}$ to stress the multivariate framework,
where $\Phi_{k}$ is an $L\times L$ matrix. Again, we can restrict
$\mathcal{E}$ in a suitable way. For example, we can impose that
$\Phi_{k}$ is lower triangular. This restriction has a variety of
implications going from Granger causality to exogeneity and it is
of much interest in econometrics (e.g., Sims, 1980). For fixed $L$,
all the results in this paper apply to this problem as well, with
obvious changes if we modify the constraint to $\sum_{k=1}^{\infty}\left|\Phi_{k}\right|^{2}\lambda_{k}^{2}\leq B$
where $\left|\Phi_{k}\right|$ is any matrix norm, e.g., Frobenius:
$\left|\Phi_{k}\right|=\sqrt{Trace\left(\Phi_{k}^{T}\Phi_{k}\right)}$,
where $\Phi_{k}^{T}$ is the transpose of $\Phi_{k}$. 

An extension, which does not follow directly from the results derived
here, is to consider the case where $L\rightarrow\infty$. This is
the problem where we have a large cross-section ($L$ is the dimensional
of the vector $Y_{t}$ in (\ref{EQ_VARModel})). In this case, the
constraint cannot use an arbitrary matrix norm (norms are not equivalent
in infinite dimensional spaces). Results in Lutz and Bühlmann (2006)
together with the ones derived here can provide initial guidance on
how to tackle this problem in the future.

\section{Proofs\label{Section_proofs}}

At first we include the short proof of Lemma \ref{Lemma_ergodicityModel}

\begin{proof}  {[}Lemma \ref{Lemma_ergodicityModel}{]}A stationary
infinite AR process with absolutely summable AR coefficients has an
infinite MA representation with absolutely summable coefficient and
it is invertible (Lemma 2.1 in Bühlmann, 1995). Hence, there are coefficients
$\psi_{s}$'s such that $Y_{t}=\sum_{s=0}^{\infty}\psi_{s}\varepsilon_{t-s}$
and 
\[
\sum_{k=1}^{\infty}\left|\mathbb{E}Y_{t}Y_{t-k}\right|\leq\sigma^{2}\sum_{k=1}^{\infty}\sum_{s=0}^{\infty}\left|\psi_{s+k}\right|\left|\psi_{s}\right|<\infty,
\]
which means that the autocovariance function is absolutely summable.
The moment bound follows from the infinite MA representation and the
bound on the fourth moment of the innovations.\end{proof}

\subsection{Proof of Theorem \ref{Theorem_consistency}}

We divide the proof into two parts. One only concerns results under
the Euclidean norm. The other is concerned with convergence results
under the RKHS norm. 

\subsubsection{Consistency Under the Euclidean Norm}

Few lemmas are needed for the proof. Throughout, we shall use the
notation $X_{t}\left(a\right)=\sum_{k=1}^{\infty}a_{k}Y_{t-k}$ for
any $a\in\mathbb{R}^{\infty}$.

\begin{lemma}\label{Lemma_supb}For $\rho:=\left(2\lambda+1\right)/2>1$
$(\lambda>1/2$ as in Condition \ref{Condition_ellipsoid}) and real
constants $w_{k}$'s, $\sup_{b\in\mathcal{E}_{K}\left(B\right)}\left|\sum_{k=1}^{K}b_{k}w_{k}\right|\lesssim\sum_{k=1}^{K}k^{-\rho}\left|w_{k}\right|,$
and similarly, for real constants $w_{k,l}$'s, $\sup_{b\in\mathcal{E}_{K}\left(B\right)}\left|\sum_{k,l=1}^{K}b_{k}b_{l}w_{lk}\right|\lesssim\sum_{k,l=1}^{K}k^{-\rho}l^{-\rho}\left|w_{kl}\right|$
.\end{lemma}

\begin{proof}Note that $\left|\sum_{k=1}^{K}b_{k}w_{k}\right|\leq\sum_{k,l=1}^{K}\frac{\left|b_{k}\right|}{k^{-\rho}}k^{-\rho}\left|w_{k}\right|$.
Given that $b\in\mathcal{E}_{K}\left(B\right)$, then $b_{k}\lesssim k^{-\rho}$
uniformly in $b\in\mathcal{E}_{K}\left(B\right)$, by Lemma \ref{Lemma_bDecay}.
This implies that the previous quantity is bounded by a constant multiple
of $\sum_{k=1}^{K}k^{-\rho}\left|w_{k}\right|$. The same argument
proves the second statement in the lemma\end{proof}

The $w_{kl}$'s in the lemma above will be partial sums of cross products
of $Y_{t}$'s, which we bound using the following. 

For arbitrary $\tau>0$, the first order conditions that define (\ref{EQ_penalizedObjFunc})
imply that 
\begin{equation}
b_{n,\tau,k}=-\frac{1}{2\tau\lambda_{k}^{2}}\frac{1}{n}\sum_{t=1}^{n}\left(Y_{t}-X_{t}\left(b_{n,\tau}\right)\right)Y_{t-k}\label{EQ_btauRepresentation}
\end{equation}
where $b_{n,\tau,k}$ is the $k^{th}$ element in $b_{n,\tau}$. By
Condition \ref{Condition_ellipsoid}, multiplying both sides by $2\tau\lambda_{k}^{2}a_{k}$
and summing over $k$,
\begin{align}
\left|\frac{1}{n}\sum_{t=1}^{n}\left(Y_{t}-X_{t}\left(b_{n,\tau}\right)\right)X_{t}\left(a\right)\right| & =2\tau\left|\sum_{k=1}^{K}\lambda_{k}^{2}b_{n,\tau,k}a_{k}\right|\nonumber \\
 & \leq2\tau\sqrt{\sum_{k=1}^{K}\lambda_{k}^{2}b_{n,\tau,k}^{2}}\sqrt{\sum_{k=1}^{K}\lambda_{k}^{2}a_{k}^{2}},\label{EQ_firstOrderCondBound}
\end{align}
recalling the definition of $X_{t}\left(a\right)$ and using the Cauchy-Schwarz
inequality. If $a\in\mathcal{E}_{K}\left(1\right)$, $\sqrt{\sum_{k=1}^{K}\lambda_{k}^{2}a_{k}^{2}}\leq1$
and the above display clearly holds uniformly in $a$. We need to
show that there is a $\tau=\tau_{n}=O_{p}\left(n^{-1/2}\right)$ such
$\sqrt{\sum_{k=1}^{K}\lambda_{k}^{2}b_{n,\tau,k}^{2}}<B$. This will
imply the display in the statement of the lemma.

\begin{lemma}\label{Lemma_sampleCovConvergenceRate}Under Condition
\ref{Condition_ellipsoid},

\[
\sup_{n,k,l>0}\mathbb{E}\left|\frac{1}{\sqrt{n}}\sum_{t=1}^{n}\left(1-\mathbb{E}\right)Y_{t-k}Y_{t-l}\right|^{2}<\infty.
\]

\end{lemma}

\begin{proof}From the proof of Lemma \ref{Lemma_ergodicityModel},
there are absolutely summable coefficients $\psi_{u}$'s, such that
$Y_{t}=\sum_{u=0}^{\infty}\psi_{u}\varepsilon_{t-u}$. For ease of
notation suppose that the i.i.d. innovations have variance one and
the MA coefficients are non-negative. By stationarity, 
\[
\mathbb{E}\left|\frac{1}{\sqrt{n}}\sum_{t=1}^{n}\left(1-\mathbb{E}\right)Y_{t-k}Y_{t-l}\right|^{2}\leq2\sum_{s=0}^{n}\mathbb{E}\left[\left(1-\mathbb{E}\right)Y_{t-k}Y_{t-l}\right]\left[\left(1-\mathbb{E}\right)Y_{t-s-k}Y_{t-s-l}\right],
\]
where the r.h.s. holds for any $t$. If we showed that 
\[
\mathbb{E}\left[\left(1-\mathbb{E}\right)Y_{t-k}Y_{t-l}\right]\left[\left(1-\mathbb{E}\right)Y_{t-s-k}Y_{t-s-l}\right]\lesssim\psi_{s}
\]
the result would follow by summability of the coefficients. To show
the above, with no loss of generality, by symmetry, consider only
the case $l\geq k$. This implies that
\begin{eqnarray*}
 &  & \mathbb{E}\left[\left(1-\mathbb{E}\right)Y_{t-k}Y_{t-l}\right]\left[\left(1-\mathbb{E}\right)Y_{t-s-k}Y_{t-s-l}\right]\\
 & = & Cov\left(Y_{t-k}Y_{t-l},Y_{t-s-k}Y_{t-s-l}\right)\\
 & = & \mathbb{E}\sum_{u_{1}=0}^{\infty}\sum_{u_{2}=0}^{\infty}\psi_{u_{1}}\psi_{u_{2}}\left[\left(1-\mathbb{E}\right)\varepsilon_{t-k-u_{1}}\varepsilon_{t-l-u_{2}}\right]\\
 &  & \times\sum_{u_{3}=0}^{\infty}\sum_{u_{4}=0}^{\infty}\psi_{u_{3}}\psi_{u_{4}}\left[\left(1-\mathbb{E}\right)\varepsilon_{t-s-k-u_{3}}\varepsilon_{t-s-l-u_{4}}\right].
\end{eqnarray*}
The above is equal to 
\[
\sum_{u_{1}=0}^{\infty}\sum_{u_{2}=0}^{\infty}\sum_{u_{3}=0}^{\infty}\sum_{u_{4}=0}^{\infty}\psi_{u_{1}}\psi_{u_{2}}\psi_{u_{3}}\psi_{u_{4}}Cov\left(\varepsilon_{t-k-u_{1}}\varepsilon_{t-l-u_{2}},\varepsilon_{t-s-k-u_{3}}\varepsilon_{t-s-l-u_{4}}\right).
\]
By the i.i.d. condition on the innovations, the covariance is zero
if the indexes are not constrained in the following sets $\left\{ k+u_{1}=l+u_{2},\,k+u_{3}=l+u_{4}\right\} $,
$\left\{ u_{1}=u_{3}+s,\,u_{2}=u_{4}+s\right\} $, $\left\{ k+u_{1}=l+u_{4}+s,\,l+u_{2}=k+u_{3}+s\right\} $.
Hence, we can consider summation with indexes in these sets only.
Splitting the sum according to the above index sets, we have respectively,
\[
I=\sum_{u=0}^{\infty}\sum_{v=0}^{\infty}\psi_{u+l-k}\psi_{u}\psi_{v+l-k}\psi_{v}Cov\left(\varepsilon_{0}^{2},\varepsilon_{u-\left(s+v\right)}^{2}\right),
\]
\[
II=\sum_{u=0}^{\infty}\sum_{v=0}^{\infty}\psi_{u+s}\psi_{v+s}\psi_{u}\psi_{v}\mathbb{E}\varepsilon_{0}^{2}\varepsilon_{\left(u-v\right)+\left(k-l\right)}^{2},
\]
\[
III=\sum_{u=0}^{\infty}\sum_{v=0}^{\infty}\psi_{u+s+\left(l-k\right)}\psi_{v+s+\left(k-l\right)}\psi_{u}\psi_{v}\mathbb{E}\varepsilon_{0}^{2}\varepsilon_{\left(u-v-s\right)+\left(k-l\right)}^{2}.
\]
By elementary change of indexes, 
\begin{eqnarray*}
I & \leq & \sum_{u=0}^{\infty}\sum_{v=0}^{\infty}\psi_{u+l-k}\psi_{u}\psi_{v+l-k}\psi_{v}1_{\left\{ u-v=s\right\} }\leq\sum_{v=0}^{\infty}\psi_{v+s+l-k}\psi_{v+s}\psi_{v+\left(l-k\right)}\psi_{v}\\
 & \leq & 2\sum_{v=0}^{\infty}\psi_{v+s}^{2}\psi_{v}^{2}\lesssim\psi_{s}^{2}.
\end{eqnarray*}
Similarly, deduce that 
\[
II\lesssim\left(\sum_{u=0}^{\infty}\psi_{u}\psi_{u+s}\right)^{2}\leq\psi_{s}^{2}\left(\sum_{u=0}^{\infty}\psi_{u}\right)^{2}\lesssim\psi_{s}^{2}.
\]
Finally, 
\[
III\lesssim\sum_{u=0}^{\infty}\sum_{v=0}^{\infty}\psi_{u}\psi_{v}\psi_{u+s+\left(l-k\right)}\psi_{v+s+\left(k-l\right)}\leq\psi_{s}\left(\sum_{u=0}^{\infty}\sum_{v=0}^{\infty}\psi_{v}\psi_{u}\right)\lesssim\psi_{s}.
\]
The bounds do not depend on $k,l$ beyond the fact that $l\geq k$.
Repeating the argument for $k>l$, the result follows. \end{proof}

Lemma \ref{Lemma_sampleCovConvergenceRate} will be used to bound
quantities such as the following 
\begin{align*}
 & \mathbb{E}\left|\sum_{k,l=1}^{\infty}k^{-\left(2\lambda+1\right)/2}l^{-\left(2\lambda+1\right)/2}\frac{1}{n}\sum_{t=1}^{n}\left(1-\mathbb{E}\right)Y_{t-k}Y_{t-l}\right|\\
\leq & \sum_{k,l=1}^{\infty}k^{-\left(2\lambda+1\right)/2}l^{-\left(2\lambda+1\right)/2}\mathbb{E}\left|\frac{1}{n}\sum_{t=1}^{n}\left(1-\mathbb{E}\right)Y_{t-k}Y_{t-l}\right|\\
\lesssim & \frac{1}{\sqrt{n}}\max_{k,l>0}\mathbb{E}\left|\frac{1}{\sqrt{n}}\sum_{t=1}^{n}\left(1-\mathbb{E}\right)Y_{t-k}Y_{t-l}\right|,
\end{align*}
where the second inequality follows because $\left(2\lambda+1\right)/2>1$.
Then, by Lemma \ref{Lemma_sampleCovConvergenceRate} the expectation
is finite because $\mathbb{E}\left|\cdot\right|\leq\left(\mathbb{E}\left|\cdot\right|^{2}\right)^{1/2}$
and it is independent of $k,l$ by stationarity. In consequence the
display is $O_{p}\left(n^{-1/2}\right)$ because convergence in $L_{1}$
implies convergence in probability. 

To establish convergence rates we need two stochastic equicontinuity
results.

\begin{lemma}\label{Lemma_equicontinuityXX}Under Condition \ref{Condition_ellipsoid},
for any $\epsilon>0$

\begin{equation}
\mathbb{E}\sup_{a,b\in\mathcal{E}\left(2B\right),\left|b\right|_{2}\leq\delta}\left|\frac{1}{\sqrt{n}}\sum_{t=1}^{n}\left(1-\mathbb{E}\right)X_{t}\left(b\right)X_{t}\left(a\right)\right|\lesssim\delta^{\frac{2\lambda-\epsilon-1}{2\lambda-\epsilon}}.\label{EQ_equicontinuityMomentEqNuisance}
\end{equation}

\end{lemma}

\begin{proof}By the triangle inequality, (\ref{EQ_equicontinuityMomentEqNuisance})
is bounded by
\[
\mathbb{E}\sup_{a,b\in\mathcal{E}\left(2B\right),\left|b\right|_{2}\leq\delta}\sum_{l=1}^{\infty}\left|a_{l}\right|\sum_{k=1}^{\infty}\left|b_{k}\right|\left|\frac{1}{\sqrt{n}}\sum_{t=1}^{n}\left(1-\mathbb{E}\right)Y_{t-k}Y_{t-l}\right|.
\]
By Lemma \ref{Lemma_supb}, there is a $\rho>1$ such that the above
is bounded by a constant multiple of 
\begin{eqnarray*}
 &  & \sum_{l=1}^{\infty}l^{-\rho}\mathbb{E}\sup_{b\in\mathcal{E}\left(2B\right),\left|b\right|_{2}\leq\delta}\sum_{k=1}^{\infty}\left|b_{k}\right|\left|\frac{1}{\sqrt{n}}\sum_{t=1}^{n}\left(1-\mathbb{E}\right)Y_{t-k}Y_{t-l}\right|\\
 & \lesssim & \sup_{l>0}\mathbb{E}\sup_{b\in\mathcal{E}\left(2B\right),\left|b\right|_{2}\leq\delta}\sum_{k=1}^{\infty}\left|b_{k}\right|\left|\frac{1}{\sqrt{n}}\sum_{t=1}^{n}\left(1-\mathbb{E}\right)Y_{t-k}Y_{t-l}\right|
\end{eqnarray*}
by summability of $l^{-\rho}$. For any positive $V$, the above display
can be written as 
\[
\sup_{l>0}\mathbb{E}\sup_{b\in\mathcal{E}\left(2B\right),\left|b\right|_{2}\leq\delta}\left(\sum_{k\leq V}+\sum_{k>V}\right)\left|b_{k}\right|\left|\frac{1}{\sqrt{n}}\sum_{t=1}^{n}\left(1-\mathbb{E}\right)Y_{t-k}Y_{t-l}\right|.
\]
 We shall bound the two sums separately. By the Cauchy-Schwarz inequality,
the first sum is bounded by 
\begin{equation}
\sqrt{\sup_{l>0}\sup_{\left|b\right|_{2}\leq\delta}\sum_{k\leq V}b_{k}^{2}\sum_{k\leq V}\mathbb{E}\left|\frac{1}{\sqrt{n}}\sum_{t=1}^{n}\left(1-\mathbb{E}\right)Y_{t-k}Y_{t-l}\right|^{2}}\lesssim\delta\sqrt{V},\label{EQ_firstSumSEBound-1}
\end{equation}
where the inequality uses Lemma \ref{Lemma_sampleCovConvergenceRate}
and $\left|b\right|_{2}\leq\delta$. Having set $V$ to such finite
value, by the Cauchy-Schwarz inequality, the second sum is bounded
by 
\begin{eqnarray*}
 &  & \sqrt{\left(\sup_{b\in\mathcal{E}\left(2B\right)}\sum_{k>V}b_{k}^{2}k^{\left(1+\epsilon\right)}\right)\left(\sup_{l>0}\sum_{k>V}k^{-\left(1+\epsilon\right)}\mathbb{E}\left|\frac{1}{\sqrt{n}}\sum_{t=1}^{n}\left(1-\mathbb{E}\right)Y_{t-k}Y_{t-l}\right|^{2}\right)}\\
 & \lesssim & \sqrt{V^{\left(1+\epsilon\right)}\lambda_{V}^{-2}\left(\sup_{b\in\mathcal{E}\left(2B\right)}\sum_{k>V}b_{k}^{2}\lambda_{k}^{2}\right)}
\end{eqnarray*}
for any $\epsilon\in\left(0,2\lambda-1\right)$, using again Lemma
\ref{Lemma_sampleCovConvergenceRate}, and the fact that $k^{-\left(1+\epsilon\right)}$
is summable and $k^{\left(1+\epsilon\right)}\lambda_{k}^{-2}$ is
decreasing. The r.h.s. is then bounded by a constant multiple of $V^{\left(1+\epsilon-2\lambda\right)/2}$.
Equating $\delta\sqrt{V}$ with $V^{\left(1+\epsilon-2\lambda\right)/2}$
we choose $V=\delta^{2/\left(2\lambda-\epsilon\right)}$, implying
that $\delta\sqrt{V}+V^{\left(1+\epsilon-2\lambda\right)/2}\lesssim\delta^{\frac{2\lambda-\epsilon-1}{2\lambda-\epsilon}}$
and the lemma is proved. \end{proof}

\begin{lemma}\label{Lemma_equicontinuityEpsilonX}Under Condition
\ref{Condition_ellipsoid}, for any $\epsilon>0$,
\[
\mathbb{E}\sup_{b\in\mathcal{E}\left(2B\right),\left|b\right|_{2}\leq\delta}\left|\frac{1}{\sqrt{n}}\sum_{t=1}^{n}\varepsilon_{t}X_{t}\left(b\right)\right|\lesssim\delta^{\frac{2\lambda-\epsilon-1}{2\lambda-\epsilon}}
\]

\end{lemma}

\begin{proof}By linearity and the triangle inequality,

\begin{eqnarray*}
 &  & \mathbb{E}\sup_{b\in\mathcal{E}\left(2B\right),\left|b\right|_{2}\leq\delta}\left|\frac{1}{\sqrt{n}}\sum_{t=1}^{n}\varepsilon_{t}X_{t}\left(b\right)\right|\\
 & \leq & \mathbb{E}\sup_{b\in\mathcal{E}\left(2B\right),\left|b\right|_{2}\leq\delta}\sum_{k=1}^{\infty}\left|b_{k}\right|\left|\frac{1}{\sqrt{n}}\sum_{t=1}^{n}\varepsilon_{t}Y_{t-k}\right|.
\end{eqnarray*}
Note that 
\[
\sup_{k>0}\mathbb{E}\left|\frac{1}{\sqrt{n}}\sum_{t=1}^{n}\varepsilon_{t}Y_{t-k}\right|^{2}\leq\sigma^{2}\gamma\left(0\right).
\]
Hence, we can proceed exactly as in the proof of Lemma \ref{Lemma_equicontinuityXX}
to deduce the result.\end{proof}

The first part of Point 1 in the theorem will be proved in Lemma \ref{Lemma_lagrangeMultiplier}
(Section \ref{Section_ConsistencyRKHSNorm}). Hence, here we shall
only derive the convergence rate. 

Define the empirical loss function
\[
L_{n}\left(b\right):=\frac{1}{n}\sum_{t=1}^{n}\left(Y_{t}-\sum_{k=1}^{\infty}b_{k}Y_{t-k}\right)^{2}
\]
where $b\in\mathcal{E}$. When $b\in\mathcal{E}_{K}$ the sum inside
the parenthesis only runs from $1$ to $K$. The population loss is
\[
L\left(b\right):=\mathbb{E}X_{1}^{2}\left(\varphi-b\right).
\]
Define $\beta=\beta_{K}\in\mathbb{R}^{\infty}$ such that its first
$K$ entries are as in $\varphi$ and the remaining are all zero.
The consistency proof is standard (van der Vaart and Wellner, 2000,
Theorem 3.2.5) once we show the following: 
\begin{equation}
L\left(b\right)-L\left(\beta\right)\gtrsim\left|b-\beta\right|_{2}^{2},\label{EQ_lossNormRelation}
\end{equation}
\begin{equation}
\mathbb{E}\sup_{b\in\mathcal{E}_{K}\left(B\right):\left|b-\beta\right|_{2}\leq\delta}\left|\left[L_{n}\left(b\right)-L\left(b\right)\right]-\left[L_{n}\left(\beta\right)-L\left(\beta\right)\right]\right|\lesssim\frac{\delta^{\alpha}}{\sqrt{n}},\label{EQ_stochasticEquicontinuity}
\end{equation}
for some $\alpha\in\left(0,2\right)$. Then, for any sequence $r_{n}\rightarrow\infty$
satisfying $r_{n}^{1-2\alpha}\lesssim\sqrt{n}$, $L_{n}\left(b_{n}\right)\leq L_{n}\left(\beta\right)+O_{p}\left(r_{n}^{-2}\right)$
and $\left|\varphi-\beta\right|_{2}\lesssim r_{n}^{-1}$, we have
that $\left|b_{n}-\varphi\right|_{2}=O_{p}\left(r_{n}^{-1}\right)$. 

At first we verify (\ref{EQ_lossNormRelation}). Note that 
\[
L\left(b\right)-L\left(\beta\right)=\sum_{k,l=1}^{\infty}\left(b_{k}-\beta_{k}\right)\left(b_{l}-\beta_{l}\right)\gamma\left(k-l\right),
\]
where $\gamma\left(k\right)$ is the autocovariance function (ACF)
of the $Y_{t}$'s. The estimator is uniquely identified if the matrix,
say $\Gamma$, with $\left(k,l\right)$ entry equal to $\gamma\left(k-l\right)$,
is strictly positive definite with smallest eigenvalue $\theta_{min}>0$
(see remarks after Lemma 2.2. in Kreiss et al., 2011). This is the
case if the spectral density of $\left(Y_{t}\right)_{t\in\mathbb{Z}}$,
say $g\left(\omega\right)$, is bounded away from zero. The spectral
density of the AR model (\ref{EQ_trueModel}) is given by $g\left(\omega\right)=\left(2\pi\right)^{-1}\sigma^{2}/\varphi\left(\omega\right)$,
where $\varphi\left(\omega\right)=\left|\sum_{k=0}^{\infty}\varphi_{k}e^{-ik\omega}\right|^{2}$
with $\varphi_{0}:=1$. Noting that by Condition \ref{Condition_ellipsoid},
$\varphi\left(\omega\right)=\left|\sum_{k=0}^{\infty}\varphi_{k}e^{-ik\omega}\right|^{2}\leq\left(\sum_{k=0}^{\infty}\left|\varphi_{k}\right|\right)^{2}<\infty$,
deduce that the eigenvalues of $\Gamma$ are bounded away from zero.
Hence, 
\begin{equation}
L\left(b\right)-L\left(\beta\right)\ge\theta_{min}^{-1}\sum_{k=1}^{\infty}\left(b_{k}-\beta_{k}\right)^{2}=\left|b-\beta\right|_{2}^{2},\label{EQ_acfMinimumEigenvalueRelation}
\end{equation}
and (\ref{EQ_lossNormRelation}) holds.

Using the notation $Y_{t}=X_{t}\left(\varphi\right)+\varepsilon_{t}$,
the empirical loss is equal to 
\[
L_{n}\left(b\right)=\frac{1}{n}\sum_{t=1}^{n}\left[\varepsilon_{t}^{2}+X_{t}^{2}\left(\varphi-b\right)+2\varepsilon_{t}X_{t}\left(\varphi-b\right)\right].
\]
This implies that 
\begin{align*}
 & \left(L_{n}\left(b\right)-L\left(b\right)\right)-\left(L_{n}\left(\beta\right)-L\left(\beta\right)\right)\\
= & \frac{1}{n}\sum_{t=1}^{n}\left[2\varepsilon_{t}X_{t}\left(\beta-b\right)+\left(1-\mathbb{E}\right)\left(X_{t}^{2}\left(b-\varphi\right)-X_{t}^{2}\left(\beta-\varphi\right)\right)\right].
\end{align*}
To verify (\ref{EQ_stochasticEquicontinuity}), we need to bound the
above uniformly in $b\in\mathcal{E}\left(B\right)$ such that $\left|b-\beta\right|_{2}\leq\delta.$
To this end, apply Lemma \ref{Lemma_equicontinuityEpsilonX} to the
first term on the r.h.s. to find that the uniform bound is a constant
multiple of $n^{-1/2}\delta^{\frac{2\lambda-\epsilon-1}{2\lambda-\epsilon}}$
for any $\epsilon>0$. By basic algebraic manipulations, the second
term on the r.h.s. of the display is 
\begin{eqnarray*}
 &  & \left(1-\mathbb{E}\right)\left(X_{t}^{2}\left(b-\varphi\right)-X_{t}^{2}\left(\beta-\varphi\right)\right)\\
 & = & \frac{1}{n}\sum_{t=1}^{n}\left(1-\mathbb{E}\right)Y_{t-k}Y_{t-l}\left(b_{n,k}-\beta_{k}\right)\left(b_{n,l}-\varphi_{l}\right)\\
 &  & +\frac{1}{n}\sum_{t=1}^{n}\left(1-\mathbb{E}\right)Y_{t-k}Y_{t-l}\left(\beta_{k}-\varphi_{k}\right)\left(b_{n,l}-\beta_{l}\right).
\end{eqnarray*}
Note that both $\varphi-b$ and $\beta-\varphi$ are in $\mathcal{E}\left(2B\right)$.
We apply Lemma \ref{Lemma_equicontinuityXX} to deduce that each term
on the r.h.s. of the above display is uniformly bounded in $L_{1}$
by a constant multiple of $n^{-1/2}\delta^{\frac{2\lambda-\epsilon-1}{2\lambda-\epsilon}}$
for any $\epsilon>0$ when $\left|b-\beta\right|_{2}\leq\delta$.
Hence (\ref{EQ_stochasticEquicontinuity}) is verified with $\alpha=\frac{2\lambda-\epsilon-1}{2\lambda-\epsilon}$.
When we are only interested in a finite dimensional model, we can
take $\lambda\rightarrow\infty$ to deduce that $\alpha=1$, which
is the parametric case. 

To find $r_{n}$ note that 
\[
L_{n}\left(b_{n}\right)-L_{n}\left(\beta\right)\leq L_{n}\left(b_{n}\right)-\inf_{b\in\mathcal{E}_{K}\left(B\right)}L_{n}\left(b\right)=0.
\]
Also, $\left|\varphi-\beta\right|_{2}=\left(\sum_{k>K}\left|\varphi_{k}\right|^{2}\right)^{1/2}\lesssim K^{-\lambda}/\ln^{1+\epsilon}\left(K\right)$
for some $\epsilon>0$ using Lemma \ref{Lemma_bDecay} and bounding
the sum with an integral ad using the fact that $\ln^{1+\epsilon}\left(\cdot\right)$
is slowly varying at infinity. Hence we deduce that $r_{n}^{-1}\asymp\left(K^{-\lambda}/\ln^{1+\epsilon}\left(K\right)\right)+n^{-\frac{1}{2}\left(\frac{2\lambda-\epsilon}{2\lambda-\epsilon+1}\right)}$
as stated in Point 1 of the theorem.

\subsubsection{Consistency Under the RKHS Norm \label{Section_ConsistencyRKHSNorm}}

The proof depends on a few preliminary lemmas. Let $\varphi_{\tau}=\varphi_{K,\tau}\in\mathcal{E}_{K}$
be the penalized population estimator

\begin{equation}
\varphi_{\tau}=\arg\inf_{b\in\mathcal{E}_{K}}\mathbb{E}X_{1}^{2}\left(b-\varphi\right)+\tau\left|b\right|_{\mathcal{E}}^{2}.\label{EQ_populationPenalizedEstimator}
\end{equation}
The following can be deduced from Theorem 5.9 in Steinwart and Christmann
(2008, eq. 5.14). The proof is given, as the context might seem different
at first sight.

\begin{lemma}\label{Lemma_SteinwartChristmann}Suppose Condition
\ref{Condition_ellipsoid}. For arbitrary but fixed $\tau>0$, consider
$b_{n,\tau}$ and $\varphi_{\tau}$ in (\ref{EQ_penalizedObjFunc})
and (\ref{EQ_populationPenalizedEstimator}) with $K$ possibly diverging
to infinity. Then, 
\[
\sqrt{\sum_{k=1}^{K}\lambda_{k}^{2}\left(b_{n,\tau,k}-\varphi_{\tau,k}\right)^{2}}\leq\sqrt{\sum_{k=1}^{K}\frac{1}{\tau^{2}\lambda_{k}^{2}}\left(\frac{1}{n}\sum_{t=1}^{n}\left(1-\mathbb{E}\right)\left(Y_{t}-X_{t}\left(\varphi_{\tau}\right)\right)Y_{t-k}\right)^{2}},
\]
where $b_{n,\tau,k}$ is the $k^{th}$ entry in the $K$ dimensional
vector $b_{n,\tau}$, and similarly for $\varphi_{\tau,k}$. \end{lemma}

\begin{proof}By convexity of the square error loss, 
\[
\frac{1}{n}\sum_{t=1}^{n}\left(Y_{t}-X_{t}\left(\varphi_{\tau}\right)\right)\left(X_{t}\left(b_{n,\tau}\right)-X_{t}\left(\varphi_{\tau}\right)\right)\leq\frac{1}{n}\sum_{t=1}^{n}\left(Y_{t}-X_{t}\left(b_{n,\tau}\right)\right)^{2}-\frac{1}{n}\sum_{t=1}^{n}\left(Y_{t}-X_{t}\left(\varphi_{\tau}\right)\right)^{2}.
\]
Note the following algebraic equality,
\[
2\tau\sum_{k=1}^{\infty}\lambda_{k}^{2}\left(b_{n,\tau,k}-\varphi_{\tau,k}\right)\varphi_{\tau,k}+\tau\sum_{k=1}^{\infty}\lambda_{k}^{2}\left(b_{n,\tau,k}-\varphi_{\tau,k}\right)^{2}=\tau\sum_{k=1}^{\infty}\lambda_{k}^{2}b_{n,\tau,k}^{2}-\tau\sum_{k=1}^{\infty}\lambda_{k}^{2}\varphi_{\tau,k}^{2}.
\]
The above two displays imply 
\begin{eqnarray*}
 &  & \frac{1}{n}\sum_{t=1}^{n}\left(Y_{t}-X_{t}\left(\varphi_{\tau}\right)\right)\left(X_{t}\left(b_{n,\tau}\right)-X_{t}\left(\varphi_{\tau}\right)\right)\\
 &  & +2\tau\sum_{k=1}^{\infty}\lambda_{k}^{2}\left(b_{n,\tau,k}-\varphi_{\tau,k}\right)\varphi_{\tau,k}+\tau\sum_{k=1}^{\infty}\lambda_{k}^{2}\left(b_{n,\tau,k}-\varphi_{\tau,k}\right)^{2}\\
 & \leq & \frac{1}{n}\sum_{t=1}^{n}\left(Y_{t}-X_{t}\left(b_{n,\tau}\right)\right)^{2}+\tau\sum_{k=1}^{\infty}\lambda_{k}^{2}b_{n,\tau,k}^{2}-\frac{1}{n}\sum_{t=1}^{n}\left(Y_{t}-X_{t}\left(\varphi_{\tau}\right)\right)^{2}-\tau\sum_{k=1}^{\infty}\lambda_{k}^{2}\varphi_{\tau,k}^{2}\leq0
\end{eqnarray*}
where the most r.h.s. follows because $b_{n,\tau}$ minimizes the
empirical penalized risk. The first order conditions for $\varphi_{\tau}$
read 
\begin{equation}
\varphi_{\tau,k}=-\frac{1}{2\tau\lambda_{k}^{2}}\mathbb{E}\left(Y_{t}-X_{t}\left(\varphi_{\tau}\right)\right)Y_{t-k}\label{EQ_populationSolutionMomentEq}
\end{equation}
for $k\geq1$. Substituting this in the previous display,
\begin{eqnarray*}
 &  & \frac{1}{n}\sum_{t=1}^{n}\left(Y_{t}-X_{t}\left(\varphi_{\tau}\right)\right)\left(X_{t}\left(b_{n,\tau}\right)-X_{t}\left(\varphi_{\tau}\right)\right)\\
 &  & -\mathbb{E}\left(Y_{t}-X_{t}\left(\varphi_{\tau}\right)\right)\sum_{k=1}^{K}\left(b_{n,\tau,k}-\varphi_{\tau,k}\right)Y_{t-k}+\tau\sum_{k=1}^{K}\lambda_{k}^{2}\left(b_{n,\tau,k}-\varphi_{\tau,k}\right)^{2}\leq0.
\end{eqnarray*}
Rearranging and using the definition of $X_{t}\left(b_{n,\tau}-\varphi_{\tau}\right)$,
deduce that
\begin{eqnarray*}
 &  & \tau\sum_{k=1}^{K}\lambda_{k}^{2}\left(b_{n,\tau,k}-\varphi_{\tau,k}\right)^{2}\\
 & \leq & \frac{1}{n}\sum_{t=1}^{n}\left(\mathbb{E}-1\right)\left(Y_{t}-X_{t}\left(\varphi_{\tau}\right)\right)\sum_{k=1}^{K}\left(b_{n,\tau,k}-\varphi_{\tau,k}\right)Y_{t-k}\\
 & \leq & \sqrt{\sum_{k=1}^{K}\frac{1}{\lambda_{k}^{2}}\left(\frac{1}{n}\sum_{t=1}^{n}\left(\mathbb{E}-1\right)\left(Y_{t}-X_{t}\left(\varphi_{\tau}\right)\right)Y_{t-k}\right)^{2}}\sqrt{\sum_{k=1}^{K}\lambda_{k}^{2}\left(b_{n,\tau,k}-\varphi_{\tau,k}\right)^{2}},
\end{eqnarray*}
using the Cauchy-Schwarz inequality in the last step. This implies
the result of the lemma after simple rearrangement. \end{proof}

The next lemma establishes the relation between the constrained and
penalized estimator and states a bound for the distance between the
sample and population penalized estimator under the RKHS norm.

\begin{lemma}\label{Lemma_lagrangeMultiplier}Suppose that $\varphi\in\mathrm{int}\left(\mathcal{E}\left(B\right)\right)$.
Under Condition \ref{Condition_ellipsoid}, if $a\in\mathcal{E}_{K}\left(1\right)$,
and $b_{n,\tau}$ is as in (\ref{EQ_penalizedObjFunc}), there is
$\tau=\tau_{n}=O_{p}\left(n^{-1/2}\right)$ such that $\left|b_{n,\tau}\right|_{\mathcal{E}}<B$
and 
\[
\frac{1}{\sqrt{n}}\sum_{t=1}^{n}\left(Y_{t}-X_{t}\left(b_{n,\tau}\right)\right)X_{t}\left(a\right)=O_{p}\left(B\sqrt{\sum_{k=1}^{K}\lambda_{k}^{2}a_{k}^{2}}\right),
\]
where the above bound holds uniformly in $a\in\mathcal{E}_{K}\left(1\right)$.
In consequence, there is a $\tau=O_{p}\left(n^{-1/2}\right)$ such
that $b_{n,\tau}=b_{n}$.

Moreover, for any $\tau>0$,
\[
\sqrt{\sum_{k=1}^{K}\frac{1}{\tau^{2}\lambda_{k}^{2}}\left(\frac{1}{n}\sum_{t=1}^{n}\left(1-\mathbb{E}\right)\left(Y_{t}-X_{t}\left(\varphi_{\tau}\right)\right)Y_{t-k}\right)^{2}}=O_{p}\left(\tau^{-1}n^{-1/2}\right).
\]
\end{lemma}

\begin{proof}Suppose that $\tau>0$ as otherwise, by the first order
conditions, the r.h.s. in the first display in the statement of lemma
is exactly zero and there is nothing to prove. 

By the triangle inequality, 
\begin{equation}
\sqrt{\sum_{k=1}^{K}\lambda_{k}^{2}b_{n,\tau,k}^{2}}\leq\sqrt{\sum_{k=1}^{K}\lambda_{k}^{2}\varphi_{\tau,k}^{2}}+\sqrt{\sum_{k=1}^{K}\lambda_{k}^{2}\left(b_{n,\tau,k}-\varphi_{\tau,k}\right)^{2}}.\label{EQ_sum_b_decomposition}
\end{equation}
For $\tau\geq0$, $\sqrt{\sum_{k=1}^{K}\lambda_{k}^{2}\varphi_{\tau,k}^{2}}\leq\sqrt{\sum_{k=1}^{K}\lambda_{k}^{2}\varphi_{k}^{2}}$
, as the penalized population estimator must have norm no larger than
$\varphi$. By this remark and the fact that $\varphi\in\mathrm{int}\left(\mathcal{E}\left(B\right)\right)$,
there is an $\epsilon>0$ such that the first term on the r.h.s. is
$B-3\epsilon$. Lemma \ref{Lemma_SteinwartChristmann} gives 
\begin{eqnarray}
 &  & \sum_{k=1}^{K}\lambda_{k}^{2}\left(b_{n,\tau,k}-\varphi_{\tau,k}\right)^{2}\nonumber \\
 & \leq & \sum_{k=1}^{K}\frac{1}{\tau^{2}\lambda_{k}^{2}}\left[\frac{1}{n}\sum_{t=1}^{n}\left(Y_{t}-X_{t}\left(\varphi_{\tau}\right)\right)Y_{t-k}-\mathbb{E}\left(Y_{t}-X_{t}\left(\varphi_{\tau}\right)\right)Y_{t-k}\right]^{2}.\label{EQ_steinwarthChristmannBound}
\end{eqnarray}
Adding and subtracting $\left(1-\mathbb{E}\right)X_{t}\left(\varphi\right)Y_{t-k}$,
and then using the basic inequality $\left(x+y\right)^{2}\leq2x^{2}+2y^{2}$
for any real $x,y$, the r.h.s. is 
\begin{eqnarray*}
 &  & \sum_{k=1}^{K}\frac{1}{\tau^{2}\lambda_{k}^{2}}\left[\frac{1}{n}\sum_{t=1}^{n}\left(1-\mathbb{E}\right)\left(Y_{t}-X_{t}\left(\varphi\right)\right)Y_{t-k}+\frac{1}{n}\sum_{t=1}^{n}\left(1-\mathbb{E}\right)\left(X_{t}\left(\varphi\right)-X_{t}\left(\varphi_{\tau}\right)\right)Y_{t-k}\right]^{2}.\\
 & \leq & 2\sum_{k=1}^{K}\frac{1}{\tau^{2}\lambda_{k}^{2}}\left[\frac{1}{n}\sum_{t=1}^{n}\left(1-\mathbb{E}\right)\left(Y_{t}-X_{t}\left(\varphi\right)\right)Y_{t-k}\right]^{2}\\
 &  & +2\sum_{k=1}^{K}\frac{1}{\tau^{2}\lambda_{k}^{2}}\left[\frac{1}{n}\sum_{t=1}^{n}\left(1-\mathbb{E}\right)\left(X_{t}\left(\varphi\right)-X_{t}\left(\varphi_{\tau}\right)\right)Y_{t-k}\right]^{2}.
\end{eqnarray*}
Recalling that our goal is to bound the second term on the r.h.s.
of (\ref{EQ_sum_b_decomposition}), the above two displays imply that
\begin{eqnarray}
\sqrt{\sum_{k=1}^{K}\lambda_{k}^{2}\left(b_{n,\tau,k}-\varphi_{\tau,k}\right)^{2}} & \leq & \frac{1}{\tau}\sqrt{2\sum_{k=1}^{K}\frac{1}{\lambda_{k}^{2}}\left[\frac{1}{n}\sum_{t=1}^{n}\left(1-\mathbb{E}\right)\left(Y_{t}-X_{t}\left(\varphi\right)\right)Y_{t-k}\right]^{2}}\nonumber \\
 &  & +\frac{1}{\tau}\sqrt{2\sum_{k=1}^{K}\frac{1}{\lambda_{k}^{2}}\left[\frac{1}{n}\sum_{t=1}^{n}\left(1-\mathbb{E}\right)\left(X_{t}\left(\varphi\right)-X_{t}\left(\varphi_{\tau}\right)\right)Y_{t-k}\right]^{2}}\nonumber \\
 & =: & I+II.\label{EQ_b_phi_coefficientsSum}
\end{eqnarray}
To bound $I$ on the r.h.s. note that for $k>0$,
\begin{eqnarray*}
\mathbb{E}\left[\frac{1}{n}\sum_{t=1}^{n}\left(1-\mathbb{E}\right)\left(Y_{t}-X_{t}\left(\varphi\right)\right)Y_{t-k}\right]^{2} & = & \mathbb{E}\left[\frac{1}{n}\sum_{t=1}^{n}\varepsilon_{t}Y_{t-k}\right]^{2}\\
 & = & \frac{\sigma^{2}\gamma\left(0\right)}{n}
\end{eqnarray*}
(recall $\gamma\left(k\right)$ is the ACF) so that
\[
\sum_{k=1}^{K}\frac{1}{\lambda_{k}^{2}}\left[\frac{1}{n}\sum_{t=1}^{n}\left(1-\mathbb{E}\right)\left(Y_{t}-X_{t}\left(\varphi\right)\right)Y_{t-k}\right]^{2}=O_{p}\left(\frac{\sigma^{2}\gamma\left(0\right)}{n}\right)
\]
because the coefficients $\lambda_{k}^{-2}$ are summable. Hence,
it is possible to find a $\tau=O_{p}\left(n^{-1/2}\right)$ such that
$I\leq\epsilon$. To bound $II$, recall that $\varphi_{\tau},\varphi\in\mathcal{E}\left(B\right)$
for any $\tau\geq0$, and write 
\[
W_{k,l}:=\frac{1}{\sqrt{n}}\sum_{t=1}^{n}\left(1-\mathbb{E}\right)Y_{t-l}Y_{t-k}
\]
for ease of notation. Then, for $\rho=\left(2\lambda+1\right)/2>1$,
\begin{eqnarray}
III & := & \mathbb{E}\sum_{k=1}^{K}\frac{1}{\lambda_{k}^{2}}\left[\frac{1}{n}\sum_{t=1}^{n}\left(1-\mathbb{E}\right)\left(X_{t}\left(\varphi\right)-X_{t}\left(\varphi_{\tau}\right)\right)Y_{t-k}\right]^{2}\nonumber \\
 & \leq & \sum_{k=1}^{K}\frac{1}{\lambda_{k}^{2}}\mathbb{E}\sup_{b\in\mathcal{E}\left(2B\right)}\left[\frac{1}{n}\sum_{t=1}^{n}\left(1-\mathbb{E}\right)\sum_{l=1}^{\infty}b_{l}Y_{t-l}Y_{t-k}\right]^{2}\nonumber \\
 & \leq & \frac{1}{n}\sum_{k=1}^{K}\frac{1}{\lambda_{k}^{2}}\sum_{l,j=1}^{\infty}l^{-\rho}j^{-\rho}\mathbb{E}W_{k,l}W_{k,j}\nonumber \\
 & \lesssim & \frac{1}{n}\sup_{k,l,j}\mathbb{E}W_{k,l}W_{k,j}\leq\frac{1}{n}\sup_{k,l}\mathbb{E}W_{k,l}^{2}\label{EQ_populationMomentEqDiffPenUnpen}
\end{eqnarray}
using Lemma \ref{Lemma_supb} in the second inequality and summability
of the coefficient in the last step. By Lemma \ref{Lemma_sampleCovConvergenceRate},
$\mathbb{E}W_{k,l}^{2}\leq c$ for some finite absolute constant $c$.
Hence, deduce that $III=O_{p}\left(n^{-1}\right)$, which implies
that $II=O_{p}\left(\tau^{-1}n^{-1/2}\right)$. Hence, there is a
$\tau=O_{p}\left(n^{-1/2}\right)$ such that $II\leq\epsilon$. The
control of $I+II$ implies that (\ref{EQ_b_phi_coefficientsSum})
is not greater than $2\epsilon$ for suitable $\tau$. Hence, we have
shown that there is a $\tau=O_{p}\left(n^{-1/2}\right)$ such that
(\ref{EQ_sum_b_decomposition}) is not greater than $B-\epsilon$.
This bound for (\ref{EQ_sum_b_decomposition}) together with (\ref{EQ_firstOrderCondBound})
proves the first display in the lemma. To see that this also implies
that there is a $\tau=O_{p}\left(n^{-1/2}\right)$ such that $b_{n,\tau}=b_{n}$
note that $\left|b_{n,\tau}\right|_{\mathcal{E}}$ is non-deceasing
as $\tau\rightarrow0$. Hence, $b_{n,\tau}=b_{n}$ for the smallest
$\tau$ such that $\left|b_{n,\tau}\right|_{\mathcal{E}}\leq B$ 

The last statement in the lemma follows from (\ref{EQ_steinwarthChristmannBound})
and the just derived bound for (\ref{EQ_b_phi_coefficientsSum}).\end{proof}

We now estimate the approximation error. 

\begin{lemma}\label{Lemma_approximation}For any $K\rightarrow\infty$,
we have that $\left|\varphi_{K}-\varphi_{\tau}\right|_{\mathcal{E}}\rightarrow0$
as $\tau\rightarrow0$ where $\varphi_{K}$ is as in (\ref{EQ_populationKEstimator}).
Moreover, if $\tau=O_{p}\left(K^{-2\lambda}\right)$, then $\left|\varphi_{K}-\varphi_{\tau}\right|_{\mathcal{E}}=O_{p}\left(\tau K^{2\lambda}\right)$.
\end{lemma}

\begin{proof}The first part of the lemma is just Theorem 5.17 in
Steinwart and Christmann (2008). Hence, we only need to prove the
second statement. Let $\Gamma$ be the $K\times K$ matrix with $\left(k,l\right)$
entry $\gamma\left(k-l\right)$ and let $\Gamma_{1}$ be the first
column in $\Gamma$. Let $\tilde{\varphi}_{K},\tilde{\varphi}_{\tau}\in\mathbb{R}^{K}$
to be the first $K$ entries in $\varphi_{K},\varphi_{\tau}\in\mathcal{E}_{K}$.
Recall that in both $\varphi_{K}$ and $\varphi_{\tau}$ all entries
$k>K$ are zero. Then, $\tilde{\varphi}_{K}=\Gamma^{-1}\Gamma_{1}$,
and writing $D:=\tau^{1/2}\Lambda$ for $\Lambda$ as in (\ref{EQ_objFunc}),
\[
\tilde{\varphi}_{\tau}=\left(DD+\Gamma\right)^{-1}\Gamma_{1}.
\]
By the Woodbury identity (Petersen and Pedersen, 2012, eq.159)
\[
\left(DD+\Gamma\right)^{-1}=\Gamma^{-1}-\Gamma^{-1}D\left(I+D\Gamma^{-1}D\right)^{-1}D\Gamma^{-1}
\]
we have that
\[
\tilde{\varphi}_{K}-\tilde{\varphi}_{\tau}=\left[\Gamma^{-1}D\left(I+D\Gamma^{-1}D\right)^{-1}D\Gamma^{-1}\right]\Gamma_{1}.
\]
Hence, 
\begin{align*}
\left|\varphi_{K}-\varphi_{\tau}\right|_{\mathcal{E}} & =\left|\Lambda\Gamma^{-1}D\left(I+D\Gamma^{-1}D\right)^{-1}D\Gamma^{-1}\Gamma_{1}\right|_{2}\\
 & =\left|D\Gamma^{-1}D\left(I+D\Gamma^{-1}D\right)^{-1}\Lambda\tilde{\varphi}_{K}\right|_{2}
\end{align*}
using the definitions of $\tilde{\varphi}_{K}$ and $D$. For any
square matrix $W$ and compatible vector $a$, $\left|Wa\right|_{2}\leq\sigma_{\max}^{2}\left(W\right)\left|a\right|_{2}$,
where $\sigma_{\max}^{2}\left(W\right)$ is the maximum eigenvalue
of $W$. Define $W=D\Gamma^{-1}D\left(I+D\Gamma^{-1}D\right)^{-1}$.
Given that $\varphi\in\mathcal{E}_{K}\left(B\right)$, then, $\left|\Lambda\tilde{\varphi}\right|_{2}\leq B$.
Hence, we only need to find the maximum eigenvalue of $W$ to bound
the above display. The following inequalities hold for the eigenvalues
of the product of two positive definite matrices $A$ and $C$: 
\[
\sigma_{\max}^{2}\left(A\right)\sigma_{\min}^{2}\left(C\right)\leq\sigma_{\min}^{2}\left(AC\right)\leq\sigma_{\max}^{2}\left(AC\right)\leq\sigma_{\max}^{2}\left(A\right)\sigma_{\max}^{2}\left(C\right)
\]
where $\sigma_{\max}^{2}\left(\cdot\right)$ and $\sigma_{\min}^{2}\left(\cdot\right)$
are the maximum and minimum eigenvalue of the matrix argument (Bathia,
1997, problem III.6.14, p.78). In order to derive (\ref{EQ_acfMinimumEigenvalueRelation}),
we argued that $\Gamma$ has minimum eigenvalue $\theta_{\min}$ bounded
away from zero. Hence, $D\Gamma^{-1}D$ has eigenvalues in $\left[\theta_{\min}^{-1}\tau\lambda_{1}^{2},\theta_{\min}^{-1}\tau\lambda_{K}^{2}\right]$.
The matrix $\left(I+D\Gamma^{-1}D\right)$ has eigenvalues equal to
1 plus the eigenvalues of $D\Gamma^{-1}D$. Hence deduce that $\left|\varphi_{K}-\varphi_{\tau}\right|_{\mathcal{E}}\lesssim\theta_{\min}^{-1}\tau\lambda_{K}^{2}\left(1+\theta_{\min}^{-1}\tau\lambda_{1}^{2}\right)$.
This is just $O\left(\tau\lambda_{K}^{2}\right)=O\left(\tau K^{2\lambda}\right)$
as required.\end{proof}

We need a final approximation result.

\begin{lemma}\label{Lemma_approximationBetaPhi}Recall (\ref{EQ_populationKEstimator}).
If $\varphi\in\mathcal{E}$, then $\left|\varphi-\varphi_{K}\right|_{\mathcal{E}}=1/\ln^{1+\epsilon}\left(K\right)$
as $K\rightarrow\infty$. If also $\left|\varphi_{k}\right|\lesssim k^{-\nu}$
with $\nu>\left(2\lambda+1\right)/2$, then, $\left|\varphi-\varphi_{K}\right|_{\mathcal{E}}=O\left(K^{\left(2\lambda+1-2\nu\right)/2}\right)$.
\end{lemma}

\begin{proof}Recall the definition of $\beta=\beta_{K}\in\mathbb{R}^{\infty}$
just before (\ref{EQ_lossNormRelation}). Let $\tilde{\beta}\in\mathbb{R}^{K}$
have the same first $K$ entries as as $\beta$. Write $Y_{t}=X_{t}\left(\beta\right)+\varepsilon_{K,t}$
where $\varepsilon_{K,t}=\varepsilon_{t}-X_{t}\left(\beta-\varphi\right)$.
Given that $\tilde{\varphi}_{K}$ is the population ordinary least
square estimator, using the same notation as in the proof of Lemma
\ref{Lemma_approximation}, 
\[
\tilde{\varphi}_{K}=\tilde{\beta}+\Gamma^{-1}\mathbb{E}\left(\begin{array}{c}
Y_{t-1}\\
Y_{t-2}\\
\vdots\\
Y_{t-K}
\end{array}\right)\varepsilon_{K,t}.
\]
We need to show that the second term goes to zero under the norm $\left|\cdot\right|_{\mathcal{E}}$.
Given that the innovations are i.i.d., the expectation is equal to
\[
-\mathbb{E}\left(\begin{array}{c}
Y_{t-1}\\
Y_{t-2}\\
\vdots\\
Y_{t-K}
\end{array}\right)\sum_{l=K+1}^{\infty}Y_{t-l}\varphi_{l}=-\sum_{l=1}^{\infty}\varphi_{K+l}\left(\begin{array}{c}
\gamma\left(K-1+l\right)\\
\gamma\left(K-2+l\right)\\
\vdots\\
\gamma\left(l\right)
\end{array}\right)=:\Psi.
\]
Hence, 
\[
\left|\beta-\varphi_{K}\right|_{\mathcal{E}}=\left|\Lambda\Gamma^{-1}\Psi\right|_{2}.
\]
We need to show that this converges to zero. By similar arguments
as in the proof of Lemma \ref{Lemma_approximation}, deduce that
\[
\Psi'\Gamma^{-1}\Lambda^{2}\Gamma^{-1}\Psi\leq\theta_{\min}^{-1}\Psi'\Lambda^{2}\Gamma^{-1}\Psi\leq\theta_{\min}^{-2}\Psi'\Lambda^{2}\Psi
\]
so that it is sufficient to bound the square root of the above display.
We have that 
\[
\Psi'\Lambda^{2}\Psi=\sum_{l_{1},l_{2}=1}^{\infty}\varphi_{K+l_{1}}\varphi_{K+l_{2}}\sum_{k=1}^{K}\lambda_{k}^{2}\gamma\left(K-k+l_{1}\right)\gamma\left(K-k+l_{2}\right).
\]
Note that $\max_{k\leq K}\left|\gamma\left(K-k+l\right)\right|\leq\left|\gamma\left(l\right)\right|$,
and by Lemma \ref{Lemma_bDecay} the autocovariance function is summable.
Moreover $\lambda_{k}^{2}\asymp k^{2\lambda}$. Hence, when $\left|\varphi_{K+l}\right|\lesssim K^{-\nu}$
holds true, the above display can be bounded by a constant multiple
of 
\[
K^{-2\nu}\sum_{k=1}^{K}k^{2\lambda}\lesssim K^{\left(2\lambda+1\right)-2\nu}.
\]
Finally, by definition of $\beta$, 
\[
\left|\varphi-\beta\right|_{\mathcal{E}}^{2}=\sum_{k>K}\varphi_{k}^{2}\lambda_{k}^{2}\lesssim K^{\left(2\lambda+1\right)-2\nu}.
\]
This implies that $\left|\varphi-\varphi_{K}\right|_{\mathcal{E}}=O\left(K^{\left(2\lambda+1-2\nu\right)/2}\right)$.
If we only assume that $\varphi\in\mathcal{E}$, then $\left|\varphi_{k}\right|\lesssim k^{-\left(2\lambda+1\right)/2}/\ln^{1+\epsilon}\left(1+k\right)$
for some $\epsilon>0$ by Lemma \ref{Lemma_bDecay}. Substituting
in the above display, we have a logarithmic convergence rate rather
than polynomial.\end{proof}

We can now prove Points 2-5 in Theorem \ref{Theorem_consistency}.
If $\varphi\in\mathcal{E}$, then, there is a finite $B$ such that
$\varphi\in\mathrm{int}\left(\mathcal{E}\left(B\right)\right)$. Hence,
by Lemma \ref{Lemma_SteinwartChristmann} and \ref{Lemma_lagrangeMultiplier},
deduce that $\left|b_{n,\tau}-\varphi_{\tau}\right|_{\mathcal{E}}=O_{p}\left(\tau^{-1}n^{-1/2}\right)$
and also that $\left|b_{n,\tau}\right|_{\mathcal{E}}<B$ eventually
in probability. Hence, if $\tau n^{1/2}\rightarrow\infty$ in probability,
by Lemma \ref{Lemma_approximation}, $\left|b_{n,\tau}-\varphi_{K}\right|_{\mathcal{E}}\rightarrow0$
in probability irrespective of the fact that $K\rightarrow\infty$.
By Lemma \ref{Lemma_approximationBetaPhi}, $\left|\varphi-\varphi_{K}\right|_{\mathcal{E}}\rightarrow0$
as $K\rightarrow\infty$, so that the triangle inequality gives $\left|b_{n,\tau}-\varphi\right|_{\mathcal{E}}\rightarrow0$
in probability under the sole condition that $\tau n^{1/2}+K\rightarrow\infty$
in probability. This proves Point 2.

The approximation rates in Point 3 are from Lemma \ref{Lemma_approximationBetaPhi}. 

To show Point 4, use Lemma \ref{Lemma_approximation} for the approximation
error of the penalized estimator. We need $\tau=O_{p}\left(K^{-2\lambda}\right)$
for the lemma to apply. Use Lemmas \ref{Lemma_SteinwartChristmann}
and \ref{Lemma_lagrangeMultiplier} to derive the estimation error
relative to the penalized estimator. Hence, deduce that $\left|b_{n,\tau}-\varphi_{K}\right|_{\mathcal{E}}=O_{p}\left(\tau^{-1}n^{-1/2}+\tau K^{2\lambda}\right)$.
Equating the two terms inside the $O_{p}\left(\cdot\right)$, this
quantity is $O_{p}\left(n^{-1/4}K^{\lambda}\right)$ when $\tau\asymp n^{-1/4}K^{-\lambda}$.
This choice of $\tau$ satisfies $\tau=O_{p}\left(K^{-2\lambda}\right)$
as long as $n^{-1/4}K^{\lambda}=O\left(1\right)$, as required. 

We now prove Point 5. Lemma \ref{Lemma_lagrangeMultiplier} also shows
that for the constrained problem, the Lagrange multiplier is $\tau=\tau_{n,B}=O_{p}\left(n^{-1/2}\right)$,
and the constraint is possibly binding. In fact, there is a $K$ large
enough relatively to $n$, such that the constraint needs to be binding.
Then, $\left|b_{n}\right|_{\mathcal{E}}=B$, and from Lemma \ref{Lemma_lagrangeMultiplier}
we deduce that $\tau n^{1/2}=O_{p}\left(1\right)$. Hence, if $\varphi\in\mathrm{int}\left(\mathcal{E}\left(B\right)\right)$
there is an $\epsilon>0$ such that $\left|\varphi\right|_{\mathcal{E}}=B-\epsilon$.
Then, we must have 
\begin{align*}
\left|b_{n}-\varphi\right|_{\mathcal{E}}^{2} & =\left|b_{n}\right|_{\mathcal{E}}^{2}+\left|b\right|_{\mathcal{E}}^{2}-2\left\langle b_{n},\varphi\right\rangle _{\mathcal{E}}\\
 & =\left(B^{2}+\left(B-\epsilon\right)^{2}-2\left\langle b_{n},\varphi\right\rangle _{\mathcal{E}}\right).
\end{align*}
But $\left\langle b_{n},\varphi\right\rangle _{\mathcal{E}}\leq\left|b_{n}\right|_{\mathcal{E}}\left|\varphi\right|_{\mathcal{E}}\leq B\left(B-\epsilon\right)$.
Hence, the above display is greater or equal than 
\[
B^{2}+\left(B-\epsilon\right)^{2}-2B\left(B-\epsilon\right)\geq\epsilon^{2}.
\]
This means that $b_{n}$ cannot converge under the norm $\left|\cdot\right|_{\mathcal{E}}$.

\subsection{Proof of Corollary \ref{Corollary_consistency}}

Now prove Point 1 in the corollary. By Point 4 in Theorem \ref{Theorem_consistency},
the estimation error is $o_{p}\left(1\right)$ as long as $K\asymp n^{\kappa}$
for $\kappa\in\left(0,1/4\right)$; we also require $\tau=O_{p}\left(K^{-2\lambda}\right)$
which under the condition on $K$ also satisfies $\tau n^{1/2}\rightarrow\infty$.
Point 3 in Theorem \ref{Theorem_consistency} gives an approximation
error of order $\left(\ln K\right)^{-\left(1+\epsilon\right)}=o\left(1\right)$
because $K\rightarrow\infty$. Hence, we deduce the first part of
the corollary. 

To derive Point 2, consider Point 3 in Theorem \ref{Theorem_consistency}
under the additional condition on the decay rate of the true coefficients.
Point 4 in the same theorem gives again the estimation error. From
the sum of the two errors deduce that $\left|b_{n,\tau}-\varphi\right|_{\mathcal{E}}=O_{p}\left(n^{-1/4}K^{\lambda}+K^{\left(2\lambda+1-2\nu\right)/2}\right)$.
Equating the coefficients this is $O_{p}\left(n^{-\frac{2\nu-\left(2\lambda+1\right)}{4\left(2\nu-1\right)}}\right)$
when $K=n^{\frac{1}{2\left(2\nu-1\right)}}$. Once again, the bound
on the estimation error requires that $\tau=O_{p}\left(K^{-2\lambda}\right)$.
Under the condition on $K$ this ensures that $\tau n^{1/2}\rightarrow\infty$,
which is required.

\end{document}